\documentclass[journal]{IEEEtran} 
\usepackage{amsmath,amsthm,amstext,amsfonts,amssymb}
\usepackage{graphicx,subfig,caption,color,algorithm,algorithmic}
\usepackage{personal}  
\usepackage{hyperref} 

\theoremstyle{remark}
\newtheorem{remark}{Remark}

\title{On Regret-Optimal Learning in Decentralized Multi-player Multi-armed Bandits}

\author{Naumaan Nayyar, Dileep Kalathil and Rahul Jain
\thanks{The first and third authors are with the Department of Electrical Engineering, University of Southern California, Los Angeles, CA 90089. The second author is with the Department of EECS, University of California, Berkeley. Emails: {\tt (nnayyar,rahul.jain)@usc.edu, dileep.kalathil@berkeley.edu}.}}

\begin{document}
 
\maketitle

\begin{abstract}
We consider the problem of learning in single-player and multiplayer multi-armed bandit models. Bandit problems are classes of online learning problems that capture exploration versus exploitation tradeoffs. In a multi-armed bandit model, players can pick among many arms, and each play of an arm generates an i.i.d. reward from an unknown distribution. The objective is to design a policy that maximizes the expected reward over a time horizon for a single player setting and the sum of expected rewards for the multiplayer setting. In the multiplayer setting, arms may give different rewards to different players. There is no separate channel for coordination among the players. Any attempt at communication is costly and adds to regret. We propose two decentralizable policies, $\tt E^3$ ($\tt E$-$\tt cubed$) and $\tt E^3$-$\tt TS$, that can be used in both single player and multiplayer settings. These policies are shown to yield expected regret that grows at most as $O(\log^{1+\delta} T)$ (and $O(\log T)$ under some assumption). It is well known that $O(\log T)$ is the lower bound on the rate of growth of regret even in a centralized case. The proposed algorithms improve on prior work where regret grew at $O(\log^2 T)$. More fundamentally, these policies address the question of additional cost incurred in decentralized online learning, suggesting that there is at most an $\delta$-factor cost in terms of order of regret. This solves a problem of relevance in many domains and had been open for a while.
\end{abstract}


\section{Introduction}
\label{sec:intro}

Multi-armed bandit (MAB) models represent an exploration versus exploitation trade-off where the player must choose between \textit{exploring} the environment to find better options, and \textit{exploiting} based on her current knowledge to maximize her utility. These models are widely applicable in many application like  display advertisements, sensor networks, route planning and  spectrum sharing.  The model can be understood through a simple game of choosing between two coins with unknown biases. The coins are tossed repeatedly and one of them is chosen at each instant. If at a given instance, the chosen coin turns up heads, we get a reward of \$1, otherwise we get no reward. It is known that one of the two coins has a better bias, but the identity of the coin is not known. The question is, what is the optimal `learning' policy that helps maximize the expected reward, i.e., to discover which coin has a better bias and at the same time maximize the cumulative reward as the game is played. Note that the player doesn't know the value of the biases as well as she has no prior probability distribution on these values. This motivates the  \textit{ non-Bayesian} setting. The formulation where the player has a prior distribution on the parameters is called Bayesian multi-armed bandits.

The idea of multi-armed bandit models dates back to Thompson \cite{Th33} and the first rigorous formulation is due to Robbins \cite{robbins1985some}. The single player multi-armed bandit problem in a non-Bayesian setting was first formulated by Lai and Robbins \cite{LaRo85}.  Any bandit policy that makes the best choice more than a constant fraction of the time is said to have sublinear \textit{regret}. Regret measures the performance of any strategy formally against the best policy that could be employed if the distribution parameters were known. It was shown in \cite{LaRo85} that there is no learning policy that asymptotically has expected regret growing slower than $O(\log T)$. A learning scheme was also constructed that asymptotically achieved this lower bound.  

This model was subsequently studied and generalized by many researchers. In~\cite{AnVaWa87a}, Anantharam et al. generalized it to the case of multiple plays, i.e., the player can pick multiple arms (or coins) when there are more than 2 arms. In \cite{Ag95}, Agrawal proposed a sample mean based index policy that asymptotically achieved $O(\log T)$ regret. For the special case of bounded support for rewards, Auer et al. \cite{AuCeFi02} introduced a simple index-based policy, ${\tt UCB}_1$, that achieved logarithmic expected regret over finite time horizons. $\tt UCB_1$ has since become the benchmark to compare new algorithms against because of its power and simplicity. 

Recently, policies based on Thompson Sampling (\texttt{TS}) \cite{Th33} have experienced a surge of interest due to their much better empirical performance \cite{ChLi12}. It is a probability matching policy which, unlike the $\tt UCB$-class of policies that use a deterministic confidence bound, draws samples from a distribution to determine which arm to play based on the probability of its being optimal. The logarithmic regret performance of the policy was not proved until very recently~\cite{AgGo12}.   \cite{kaufmann2012bayesian} introduced the Bayes-UCB algorithm which also uses use a Bayesian approach for analyzing the regret bound for stochastic bandit problems.

Deterministic sequencing algorithms which have separate exploration and exploitation phases have also appeared in the literature as an alternative to the joint exploration and exploitation approaches of UCB-like and probability matching algorithms. Noteworthy among these are the Phased Exploration and Greedy Exploitation  policy for linear bandits~\cite{RuTs10} that achieves $O(\sqrt T)$ regret in general and $O(\log(T))$ regret for finitely many linearly parametrized arms. Other noteworthy algorithms include the logarithmic regret achieving deterministic sequencing of exploration and exploitation policy, \cite{VaLiZh13} with i.i.d. setting and  \cite{liu2013learning} with Markovian setting. Single-player bandit problems have also been looked at in the PAC framework, for instance, in~\cite{EvMaMa02}. However, we restrict our attention to performance in the expected sense in this work.  

In addition to single player bandits, there has been growing interest in multiplayer learning in multi-armed bandits, motivated by distributed sensor networks, wireless spectrum sharing and in particular cognitive radio networks. Suppose there are two wireless users trying to choose between two wireless channels. Each wireless channel is random, and looks different to each user. If channel statistics were known, we would try to determine a matching wherein the expected sum-rate of the two users is maximized. But the channel statistics are unknown, and they must be learnt by sampling the channels. Moreover, the two users have to do this independently and cannot share their observations as there is no dedicated communication channel between them. They, however, may communicate implicitly for coordination but this would come at the expense of reduced opportunities for rewards or benefits, and thus would add to regret. One can easily imagine a more general network setting with $M$ users and  $N$ channels. This immediately gives rise to two questions. First, what is the lower bound for decentralized learning? That is, is there an inherent cost of decentralization in such network? And second, can we design a simple learning algorithm with provably optimal performance guarantees, in the context of such a decentralized network problem? 

Policies for decentralized learning with sublinear regret have appeared in the literature for various models. When arms were restricted to have the same rewards for different users, Anandkumar et al. \cite{AnMi11} showed that logarithmic regret was achievable as the problem reduces to a ranking problem that can be solved in constant time in a decentralized manner. Similar works have also appeared for i.i.d.  \cite{VaLiZh13} \cite{LiZh10a} and Markovian \cite{liu2013learning} \cite{LaJiPo08} arm reward settings. Relaxing this assumption makes the problem more complicated as it now becomes a bipartite matching problem and no decentralized algorithm performs quick enough. In our previous work~\cite{KaNaJa14}, we proposed a policy, $\tt dUCB_4$ that achieved $O(\log^2 T)$ regret through a recurrent negotiation mechanism between players. However, the answers to the two questions above remained unknown. In a similar work \cite{avner2015learning}, authors address the problem of 
decentralized multi-armed bandits. While they address the same problem as ours, the emphasis is on  the stability  of this decentralized setting with minimum possible communication. Also, they don't provide any optimality guarantees as compared to the optimal centralized learning problem.  However, our paper assumes that players in the system remains the same. In \cite{avner2015learning} users can arrive and leave at random times.  Landgren et. al.  \cite{landgren2015distributed} uses a multi-armed bandit model for cooperative decision making problem in the context of running a consensus algorithm. Their setting is very different from the problem considered in this paper.

In this paper, we do not present an information theoretic lower bound on decentralized learning in a multiplayer multi-armed bandit problems. Such a result would be very interesting as it will also yield insight into the exact role of information sharing between players for a decentralized policy to work without an increase in expected regret. However, we managed to partially answer both questions above through two new decentralizable policies, $\tt E^3$ and $\tt E^3$-$\tt TS$, where $\tt E^3$ stands for Exponentially-spaced Exploration and Exploitation policy, which we also call as $\tt E$-$\tt cubed$.

Both policies yield expected regret of the order $O(\log^{1+\delta}T)$  ($O(\log T)$ under some assumptions) in both single and multiplayer settings.  The policies are based on exploration and exploitation in pre-determined phases such that over a long time horizon $T$, there are only logarithmically many slots in the exploration phases. It is well known that the optimal order of regret that can be achieved is $O(\log T)$ \cite{LaRo85}. These policies suggest an answer to the fundamental question of inherent cost to decentralize, that there is no cost to the order optimality, at least up to an $\log^{\delta}T$ factor. An asymptotic lower bound for the decentralized MAB problem (similar to that of the centralized MAB in \cite{LaRo85}) is an important future research question. 

%, i.e., each arm is a \textit{rested} arm but the reward sequence from it forms a Markov chain \cite{TeLi12}. 
The policies introduced in this paper, and the corresponding results hold even when the rewards are Markovian. However, we only present the i.i.d. case here and refer readers to our earlier paper~\cite{KaNaJa14} for ideas on extensions to the Markovian setting. Extensive simulations were conducted to evaluate the empirical performances of these policies and compared to prior work in the literature, including the classical $\tt UCB_1$ and \texttt{TS} policies. The decentralized policies $\tt dE^3$ and $\tt dE^3$-$\tt TS$ are compared with the previously known $\tt dUCB_4$ policy.

The rest of the paper is organized as follows. Section~\ref{sec:model} describes the model and problem formulations for single and multiplayer bandits. Section~\ref{sec:prior} describes relevant prior work in the area. The new policies $\tt E^3$ and $\tt E^3$-$\tt TS$, and their multiplayer counterparts, $\tt dE^3$ and $\tt dE^3$-$\tt TS$ are described and studied in Section~\ref{sec:new-policy}. Section~\ref{sec:simulation} presents empirical performances of new and previous policies.

\section{Model and problem formulation}
\label{sec:model}

In this section, we describe problem formulations for single and multiplayer bandits. The single player formulation has been well-studied in literature, for instance, by Auer et al. \cite{AuCeFi02} and others. The multiplayer formulation is much newer, and has appeared in our previous work \cite{KaNaJa14}.  
%where the $\tt dUCB_4$ algorithm was presented~

\subsection{Single player model}
\label{subsec:sp-model}

We consider an $N$-armed bandit problem. At each instant $t$, an arm $k$ is chosen, and a reward $X_k(t)$ is generated, from an independent and identically distributed (i.i.d.) random process with a fixed but unknown distribution. The processes are assumed to have bounded support, without loss of generality, in $[0,1]$. Arm reward distributions have means $\mu_k$ that are unknown. When choosing an arm, the player has access to the  history of rewards and actions, $\sH(t)$, with $\sH(0) := \emptyset$. Denote the arm chosen at time $t$ by $a(t) \in \sA := \{1,...,N\}$. A \textit{policy} $\alpha$ is a sequence of maps $\alpha(t): \sH(t) \rightarrow \sA$ that specifies the arm chosen at time $t$. The player's objective is to choose a policy that maximizes the expected reward over a finite time horizon $T$.

If the mean rewards of the arms were known, the problem is trivially solved by always playing the arm with the highest mean reward, i.e., $\tilde{\alpha}(t) = \arg\max_{1\leq i \leq N} \mu_i,~\forall t$. When the mean rewards are not known, the notion of \textit{regret} is used to compare policies. Regret is the difference between the  cumulative rewards obtained  by a policy $\alpha$ and when playing the most rewarding arm all the time. Formally, the player's objective is to minimize the expected regret over all causal policies $\alpha$ as defined above, which is given by,
\begin{equation}
\label{eq:sp-regret}
\sR_{\alpha}(T)=T \mu_{1}-\bbE_{\alpha}\left[\sum_{t=1}^{T} X_{\alpha(t)}(t)\right],
\end{equation}
where arm $1$ is taken to have the greatest mean w.l.o.g.

In practical implementations of bandit algorithms in low-power settings such as sensor networks where the implementation of any learning/control policy should consume minimum amount of energy, it will be useful to include a computation cost as well. This is particularly the case when the algorithms must solve combinatorial optimization problems that are NP-hard. Such costs  arise in decentralized settings in particular,  where algorithms pay a communication cost for coordination between the decentralized players. For example, as we shall see later in our decentralized learning algorithm, the players may have to spend many time slots for coming up with a bipartite matching.   We model it as a constant $C$ units of cost each time an index is computed by the policy. With this refinement, the regret of a policy $\alpha$ that computes its indices $m(T)$ times over a time horizon $T$ is,
\begin{equation}
\label{eq:regret-computation}
\tilde{\sR}_{\alpha}(T) := \mu_{1}T - \sum_{j=1}^{N} \mu_{j} \bbE_{\alpha}[n_{j}(T)] + C\bbE_{\alpha}[m(T)],
\end{equation}
where $n_j(T)$ is the number of times arm $j$ is played.

\subsection{Multiplayer model}
\label{subsec:mp-model}

We now describe the generalization of the single player, where we consider an $N$-armed bandit with $M$ players. We will refer to arms as channels interchangeably. There is no dedicated communication channel for coordination among the players. However, we do allow players to communicate with one another by playing arms in a certain way, e.g., arm 1 signals a bit `0', arm 2 can signal a bit `1'. This of course will add to regret, and hence such communication comes at a cost. We assume that $N \geq M$. 

At any instant $t$, each player choose one arm from the set of $N$ arms or takes no action (i.e., selects no arm). If more than one player picks the same arm, we regard it as a \textit{collision} and this interference results in zero reward for those players. The rest of the model is similar to the single player case. Arm $k$ chosen by player $i$ generates an i.i.d. reward $S_{i,k}(t)$ from an unknown distribution, which has bounded support, w.l.o.g., in $[0,1]$. Let $\m_{i,k}$ denote the unknown mean of $S_{i,k}(t)$.

Let $X_{i,k}(t)$ be the reward that player $i$ gets from playing arm $k$ at time $t$. Thus, if there is no collision, $X_{i,k}(t) = S_{i,k}(t)$. Denote the action of player $i$ at time $t$ by $a_i(t) \in \sA := \{1, \ldots, N \}$. Let $Y_{i}(t)$ be the communication message from player $i$ at time $t$ and $Y_{-i}(t)$ be the messages from all the other players except player $i$ at time $t$. Then, the \textit{history} seen by player $i$ at time $t$ is $\sH_i(t) = \{ (a_i(1),X_{i,a_i(1)}(1), Y_{-i}(1)),\cdots, (a_i(t-1),X_{i,a_i(t-1)}(t-1), Y_{-i}(t-1))\}$ with $\sH_i(0)=\emptyset$. A \textit{policy} $\alpha_i = (\alpha_i(t))_{t=1}^{\infty}$ for player $i$ is a sequence of maps $\alpha_i(t):\sH_i(t) \to \sA$ that specifies the arm to be played at time $t$.

The players have a \textit{team objective}: they want to maximize the expected sum of rewards $\bbE[\sum_{t=1}^T \sum_{i=1}^{M} X_{i,a_i(t)}(t)]$ over some time horizon $T$. Let $\sP(N)$ denote the set of possible permutations of the $N$ arms. If $\mu_{i,j}$ were known, the optimal policy is clearly to pick the optimal bipartite matching between arms and players (which may not be unique),
\begin{equation}
\label{eq:optmatching}
\mathbf{k}^{**} \in \arg \max_{\mathbf{k} \in \sP(N) } \sum_{i=1}^{M} \mu_{i,k_{i}}.
\end{equation}

When expected rewards are not known, players must pick learning policies that minimize the \textit{expected regret}, defined for policies $\alpha=(\alpha_i, 1 \leq i \leq M  )$ as,
\begin{equation}
\label{eq:regret}
\sR_{\alpha}(T)=T\sum_i \mu_{i,k_i^{**}}-\bbE_{\alpha}\left[\sum_{t=1}^{T} \sum_{i=1}^{M} X_{i,\alpha_i(t)}(t)\right].
\end{equation}

As in the single player model, we consider a refinement of the regret to factor in computational or communication costs.  Communication costs are justified because known distributed algorithms for bipartite matching \cite{Be92,ZaSpPa08} require a certain amount of information exchange over multiple time slots. This cost will depend on the specific algorithm. Here, however, we will just consider an `abstract'  cost $C$.

Let $C$ units of cost be incurred each time this occurs, and let $m(t)$ be the number of times it happens in time $t$. Then, the expected regret for policy $\alpha$ to be minimized is,
\begin{align}
\label{eq:dregret}  
\sR_{\alpha}(T)= &T\sum_{i=1}^{M} \mu_{i,k_i^{**}}-\bbE_{\alpha}\left[\sum_{t=1}^{T} \sum_{i=1}^{M} X_{i,\alpha_i(t)}(t)\right] \nonumber \\
&+ C \bbE_{\alpha}[m(T)].
\end{align}
where $\mathbf{k}^{**}$ is the optimal matching as defined in~\eqref{eq:optmatching}.

\section{Prior work}
\label{sec:prior}
We now briefly describe the key features and results of existing single and multiplayer bandit policies.
\subsection{Single player policies}
\label{subsec:sp-previous} 

We focus on three different MAB algorithms that capture different classes of policies.

%Of the many known policies for single-player MAB problems, we focus on three that capture different classes of policies. 

In \cite{AuCeFi02}, Auer et al. proposed an index based policy,  $\tt UCB_1$ which achieves logarithmic regret. It worked by playing the arm with the largest value of sample mean plus a confidence bound. The interval shrank deterministically as the arm got played more often and traded-off exploration and exploitation. It was shown in~\cite{AuCeFi02} that the expected regret incurred by the policy over a horizon $T$ is bounded by,
\begin{equation}
\label{regret-ucb1}
\sR_{UCB_1}(T) \leq 8\log(T) \sum_{j>1}^N \frac{1}{\Delta_j} + (1 + \frac{\pi^2}{3})\sum_{j>1}^N \Delta_j,
\end{equation}
where $\Delta_j := \mu_1 - \mu_j$.

Thompson Sampling (\texttt{TS}) is a probability-matching policy that has been around for quite some time in the literature~\cite{Th33} although it was not well-studied in the context of bandit problems until quite recently~\cite{ChLi12,AgGo12}. Arms are played randomly according to the probability of their being optimal. As an arm gets played more often, its sampling distribution become narrower. Unlike a fully Bayesian method such as Gittins Index~\cite{Gi79}, \texttt{TS} can be implemented efficiently in bandit problems. The regret of the policy was shown in ~\cite{AgGo12} to be bounded by,
\begin{equation}
\label{regret-ts}
\sR_{TS}(T) \leq O\bigg(\Big( \sum_{j>1}^N \frac{1}{\Delta_j^2} \Big)^2 \log(T)\bigg), 
\end{equation}
where the constants have been omitted for brevity. A stronger upper bound for the case of Bernoulli rewards  that matches the asymptotic rate lower bound  in Lai and Robbins \cite{LaRo85}   is given  in \cite{KaKoMu12}. Numerically, \texttt{TS} has been found to empirically outperform $\tt UCB_1$ in most settings \cite{ChLi12, KaKoMu12}.

$\tt UCB_4$ is another confidence-bound based index policy that was proposed recently to overcome some of the shortcomings of the $\tt UCB_1$ policy, namely its reliance on index computation in each time step and the difficulty in extending the algorithm to a multiplayer setting. It works by cleverly choosing a sequence of times to compute the $\tt UCB_1$ index. The expected regret was shown in ~\cite{KaNaJa14} to be bounded by,
\begin{equation}
\label{regret-ucb4}
\sR_{UCB_4}(T) \leq \Delta_{\max} \Big( \sum_{j>1}^N \frac{12\log(T)}{\Delta_j^2} + 2N \Big),
\end{equation}
where $\Delta_{\max} = \max \Delta_{j}$.

It can be shown that $\tt UCB_1$ and \texttt{TS} incur linear regret if computation cost is included in the model. The expected regret of the $\tt UCB_4$ algorithm over a time horizon $T$ with computation cost $C$ is bounded by~\cite{KaNaJa14},
\begin{equation}
\label{regret-ucb4}
\tilde{\sR}_{UCB_4}(T) \leq \Big(\Delta_{\max} + C(1 + \log(T)\Big) \Big( \sum_{j>1}^N \frac{12\log(T)}{\Delta_j^2} + 2N \Big).
\end{equation}
Thus, expected regret is $O(\log^2(T))$.

%It is worth mentioning here that the $\tt UCB_2$ policy~\cite{AuCeFi02} also achieves sub-linear expected regret with computation cost, actually outperforming $\tt UCB_4$ with a logarithmic regret growth. 

\subsection{Multiplayer policies}
\label{subsec:mp-previous}

The major issues that are encountered in decentralizing bandit policies are coordination among players and finite precision of indices being communicated. The $\tt dUCB_4$ policy~\cite{KaNaJa14} was the first such policy that did not assume identical channel rewards for different players. The policy is a natural decentralization of $\tt UCB_4$ that uses Bertsekas' auction algorithm~\cite{Be88} for distributed bipartite matching. 

If $\Delta_{\min}$ is known, the expected regret of ${\tt dUCB_4}$ is,
\begin{align*}
 \tilde{\sR}_{\tt dUCB_4}(T) &\leq  (L\Delta_{\max} + C(L)   (1+\log(T)) ) \times \\
 &\left( \frac{4 M^{3} (M+2) N \log(T) }{  (\Delta_{\min} - ( (M+1) \epsilon)^{2} }   + NM (2M+1)  \right),
\end{align*}
where $L$ is the frame length, $\Delta_{\min}$ is the minimum difference between the optimal and the next best permutations, $\Delta_{\max}$ is the maximum difference in rewards between permutations, and $\epsilon$ is the precision input of the distributed bipartite matching algorithm. Also, $C(L)$ indicates that the cost of communication and computation is a function of the frame length $L$. Thus, $\tilde{\sR}_{\tt dUCB_4}(T)= O(\log^{2}(T))$. A slight modification to the policy with increasing frame length addresses the case when $\Delta_{\min}$ is unknown~\cite{KaNaJa14}. 

\section{New (near-)logarithmic bandit policies}
\label{sec:new-policy}

In this section, we present our work in developing two closely related policies for single player bandit problems and their generalizations to multiplayer settings.

\subsection{Single player policies: $\tt E^3$ and $\tt E^3$-$\tt TS$}
\label{subsec:sp-policy}

$\tt E^3$ and $\tt E^3$-$\tt TS$ are phased policies detailed in Algorithms~\ref{algo:E^3} and~\ref{algo:E^3-b} respectively.  Their key difference from the previous policies is that they have deterministic exploration and exploitation phases. In the following, an epoch is defined to comprise of one exploration phase and one exploitation phase.

\textbf{Exploration phase:} During an exploration phase, the player tries out different arms in a round-robin fashion and computes indices for each arm. At the end of the phase, the player chooses the arm with the maximum value of the index. The index computation differs for $\tt E^3$ and $\tt E^3$-$\tt TS$ policies.

\textbf{Exploitation phase:} In this phase, the player plays the arm that was chosen at the end of the previous exploration phase. No index computation happens during the exploitation phase and the player sticks to her decisions during this phase. The length of the exploitation phase doubles each successive epoch. 

\begin{algorithm}
\caption{: Exponentially-spaced Exploration and Exploitation policy ($\tt E^3$)}
\label{algo:E^3}
\begin{algorithmic}[1]
\STATE {\bf Initialization:} Set $t=0$ and $l=1$;
\WHILE {($t \leq T$)}
\STATE {\bf Exploration Phase:}  Play each arm $j, 1 \leq j \leq N$,  $\gamma$ number of times;
\STATE Update the sample mean $\overline{X}_{j}(l)$, $1 \leq j \leq N$;
\STATE Compute the best arm $ j^{*}(l) :=  \arg \max_{1 \leq j \leq N } \overline{X}_{j}(l) $;
\STATE {\bf Exploitation Phase:}  Play arm $j^*(l)$ for $2^{l}$ time slots;
\STATE Update $t \leftarrow t+N \gamma+2^{l}$, $l \leftarrow l+1$;
\ENDWHILE
\end{algorithmic}
\end{algorithm}

\begin{algorithm}
\caption{: Exponentially-spaced Exploration and Exploitation algorithm-TS ($\tt E^3$-$\tt TS$)}
\label{algo:E^3-b}
\begin{algorithmic}[1]
\STATE {\bf Initialization:} Set $l=1$ and $t=0$. For each arm $i=1,2,...,N$, set $S_i = 0$, $F_i = 0$;
\WHILE {($t \leq T$)}
\STATE {\bf Exploration Phase:}  Play each arm $j, 1 \leq j \leq N$,  $\gamma$ number of times;
\STATE For each play of each arm $i$, store reward as $\tilde{r}_i(t)$;
\STATE Perform a Bernoulli trial with success probability $\tilde{r}_i(t)$ and observe output $r_i(t)$;
\STATE If $r_i(t) = 1$, then set $S_{i} = S_{i} + 1$, else $F_{i} = F_{i} + 1$;
\STATE Sample $\theta_i(l)$ from $Beta(S_i + 1, F_i + 1)$ distribution;
\STATE Compute the \textit{best arm} $ j^{*}(l) :=  \arg \max_{1 \leq j \leq N } \theta_{j}(l) $;
\STATE {\bf Exploitation Phase:} Play arm $j^*(l)$ for $2^{l}$ time slots;
\STATE Update $t \leftarrow t+N \gamma+2^{l}$, $l \leftarrow l+1$;
\ENDWHILE
\end{algorithmic}
\end{algorithm}

$\tt E^3$ and $\tt E^3$-$\tt TS$, while largely similar, differ in how they choose the arm to play during the exploitation phase. While $\tt E^3$ uses the simple sample mean value, $\tt E^3$-$\tt TS$ draws from a $\beta$-distribution in a manner similar to the \texttt{TS} policy.

The $\beta$-distribution is chosen in $\tt E^3$-$\tt TS$ due to convenient posterior form after Bernoulli observations.  A $\beta (a, b)$-distribution prior results in a posterior of $\beta (a+1, b)$ or $\beta (a, b+1)$ depending on success or failure of the Bernoulli trial, respectively.

We now give the performance bounds for the policies with an index computation cost $C$ in the main result of this section. Both algorithms will be analyzed concurrently as their proof techniques are largely similar. 

The following concentration inequality will be used in the analysis and is introduced here for the reader's ease.

\noindent{{\bf Fact 1:} Chernoff-Hoeffding inequality} \cite{BoLuMa13}.\\
Let $X_{1}, \ldots , X_{t}$ be a sequence of real-valued random variables, such that, for each $i \in \{1, \ldots, t \}, 0 \leq X_{i} \leq 1$ and $\mathbb{E}[X_{i}|\mathcal{F}_{i-1}] = \mu$, where $\mathcal{F}_{i} = \sigma(X_{1}, \ldots, X_{i})$. Let $S_{t}=\sum_{i=1}^{t} X_{i}$. Then for all $a \geq 0$, 
\begin{align*}
&\bbP\left( \frac{S_{t}}{t}   \geq \mu + a\right) \leq e^{-2a^{2}t},~  \bbP\left(\frac{S_{t}}{t}  \leq \mu - a\right)  \leq e^{-2a^{2}t}.
\end{align*}

We now give the main result of this section.

\begin{theorem}
\label{thm:E^3-regret}(Regret bounds for $\tt E^3$ and $\tt E^3$-$\tt TS$ policies)

Let $\Delta_{\min}$ and $\Delta_{\max}$ denote the differences between the mean rewards of the optimal arm, and the second best and worst arms, respectively. 
%\begin{enumerate}

\noindent (i)  If $\Delta_{\min}$ is known, set $\gamma = \lceil \frac{2}{\Delta_{\min}^2} \rceil$ and $\gamma_\beta = \lceil \frac{8}{\Delta_{\min}^2} \rceil$. Then, the expected regret of the $\tt E^3$ and $\tt E^3$-$\tt TS$ policies with computation cost $C$ is,
\begin{align}
&\tilde{\sR}_{\tt E^3}(T) \leq N  \Delta_{\max} \gamma  \log(T)  + N C  \log(T)  + 8 N \Delta_{\max} \\
&\tilde{\sR}_{\tt E^3-\tt TS}(T) \leq N \Delta_{\max} \gamma_\beta  \log(T)  + N C  \log(T)  + 16 N \Delta_{\max} 
\end{align} 
\noindent (ii)  If $\Delta_{\min}$ is not known, choose $\gamma=\gamma_{t}$, where $\{\gamma_{t}\}$ is  a positive sequence such that $\gamma_{t} \rightarrow \infty$ as $t \rightarrow \infty$. Then, 
\begin{align}
&\tilde{\sR}(T) \leq N \Delta_{\max} \gamma_{T} \log(T)  +  N C \log(T) + N \Delta_{\max} B 
\end{align} 
where $\tilde{\sR}(T) =  \max \{\tilde{\sR}_{\tt E^3}(T), \tilde{\sR}_{\tt E^3-TS}(T)\}$ and $B$ is a constant independent of $T$. In particular, for $\gamma_{t} = \log^{\delta}t, \delta \in (0,1)$, 
\begin{align*}
&\tilde{\sR}(T) \leq N \Delta_{\max} \log^{(1+\delta)}(T)  +  N C \log(T) + N \Delta_{\max} B(\delta) 
\end{align*} 
where $B(\delta) = 2^{l(\delta)}, l(\delta) = (\Delta^{2}_{\min}/4)^{-1/\delta} $. 
\end{theorem}
Proof is given in Appendix~\ref{app:proof-1}.

\begin{remark}
(i) For the sake of clarity, we will assume that $\gamma_t$ changes at the beginning of every exploration phase. (ii) Part 1 of the above theorem assumes the knowledge   $\Delta_{\min}$ in order to define $\gamma$. In fact we only need to know a lower bound on $\Delta_{\min}$. If $\Delta_{LB} \leq \Delta_{\min}$, we can fix $\gamma = \lceil \frac{2}{\Delta_{LB}^2} \rceil$. It is straightforward to show that, with a slight modification of the proof, the theorem still holds. Obviously, a tighter lower bound on  $\Delta_{\min}$ results in a tighter bound on the regret. 
\end{remark}

Although the bounds of $\tt E^3$ and $\tt E^3$-$\tt TS$ are poorer than $\tt UCB_1$ and $\tt TS$, they lend themselves to easy decentralization and can be extended to multiplayer bandit problems with minimal effort. Performances of all single player algorithms are compared in Section~\ref{subsec:sp-simulations}.

\subsection{Multiplayer policies: $\tt dE^3$ and $\tt dE^3$-$\tt TS$}
\label{subsec:mp-policy}

In this section, we present multiplayer generalizations of the  $\tt E^3$ and $\tt E^3$-$\tt TS$ policies that were described in the previous section. They are detailed in Algorithms~\ref{algo:dE^3} and~\ref{algo:dE^3-b}, respectively. They are also divided into exploration and exploitation phases. %We assume that all communication is error-free.

\textbf{Exploration phase:} During exploration phases, players take turns to explore arms in a round-robin fashion. 
At the end of an exploration phase, the players update their index values (either $g_{i,j}$, or $\theta_{i,j}$). Then they participate in a distributed bipartite matching to determine the players to channels assignments. This requires some additional time slots and comes at a cost, and contributes to regret.  This communication and the distributed bipartite matching process is compressed into line 5 in Algorithm~\ref{algo:dE^3} and line 8 in Algorithm~\ref{algo:dE^3-b} as a call to $\tt dBM$.  

%to the other players via a protocol for which a certain set of time-slots are set aside. For example, each player in turn could choose arm 1 for bit `0' and arm 2 for `1' to convey a certain number of bits to all the other players. These could be ``bids'' that a distributed bipartite matching algorithm needs. All this, of course, 

%This setup is illustrated in Figure~\ref{fig:dE^3}, where each subdivision of an exploration phase represents a player's turn. 
%At the end of an exploration phase, players communicate their channel preferences to one another over a predetermined channel (for e.g., channel $1$). This communication takes up time slots as it is not on a dedicated channel. When players are aware of one another's preferences, they run a distributed bipartite matching algorithm. 

\textbf{Distributed bipartite matching (dBM):} Let $g(t)$ $(g_{i,j}(t), 1 \leq i \leq M, 1 \leq j \leq N)$ denote a vector of indices. In both algorithms, $\tt dBM_{\e}$ $(g(t))$ refers to an $\epsilon$-optimal distributed bipartite matching algorithm, such as Bertsekas' auction algorithm~\cite{Be88}, that yields a matching $k^*(t) =(k_1^*(t),\ldots, k_M^*(t)) \in \sP(N) $ such that $\sum_{i=1}^{M} g_{i, k^{*}_{i}(t)}(t) \geq  \sum_{i=1}^{M}   g_{i, k_{i}}(t) ) - \epsilon, ~ \forall \mathbf{k} \in \sP(N)  , \mathbf{k} \neq \mathbf{k}^*$. The details of  $\tt dBM$ implementation is described in Section \ref{sec:dbm}

\textbf{Exploitation phase:} In this phase, players stick to the allocation given to them at the end of the distributed bipartite matching process. No index computation is carried out in this phase. The length of the exploitation phase doubles in each successive epoch. 
%\begin{figure}[tbh!]
%\includegraphics[scale=0.37]{Figures/dPEGE.pdf}
%\caption{Figure showing exploration and exploitation phases of the $\tt dE^3$ policy.} 
%\label{fig:dE^3}
%\end{figure}
\begin{algorithm}[h]
\caption{: $\tt dE^3 $}
\label{algo:dE^3}
\begin{algorithmic}[1]
\STATE {\bf Initialization:} Set $t=N \gamma$ and $l=1$.
\WHILE {($t \leq T$)}
\STATE {\bf Exploration Phase:}  Each player $i, 1 \leq i \leq M$, plays each arm $j, 1 \leq j \leq N$, $\gamma$ number of times;
\STATE Update the index $g_{i,j}(l) = \overline{X}_{i,j}(l)$;
\STATE Participate in the ${\tt dBM_{\e}}(g(l))$  algorithm to obtain a match $k^{*}(l)$;
\STATE {\bf Exploitation Phase:}  Each player $i$ plays arm $k_i^*(l)$ for $2^{l}$ time slots;
\STATE $t \leftarrow t+ MN \gamma+2^{l}$, $l \leftarrow l+1$;
\ENDWHILE
\end{algorithmic}
\end{algorithm}

\begin{algorithm}[h]
\caption{: $\tt dE^3$-$\tt TS$}
\label{algo:dE^3-b}
\begin{algorithmic}[1]
\STATE {\bf Initialization:} Set $t=N \gamma$ and $l=1$. For each arm $j=1,2,...,N$ and player $i=1,...,M$, set $S_{i,j} = 0$, $F_{i,j} = 0$;
\WHILE {($t \leq T$)}
\STATE {\bf Exploration Phase:}  Each player $i, 1 \leq i \leq M$, plays each arm $j, 1 \leq j \leq N$, $\gamma$ number of times;
\STATE For each play of each arm $j$, store reward as $\tilde{r}_{i,j}(t)$;
\STATE Perform a Bernoulli trial with success probability $\tilde{r}_{i,j}(t)$ and observe output $r_{ij}(t)$;
\STATE If $r_{ij}(t) = 1$, set $S_{i,j} = S_{i,j} + 1$, else $F_{i,j} = F_{i,j} + 1$.
\STATE Sample $\theta_{i,j}$ from $Beta(S_{i,j} + 1, F_{i,j} + 1)$ distribution.
\STATE Participate in the ${\tt dBM_{\e}}(\theta)$ algorithm to obtain a match $k^{*}(l)$;
\STATE {\bf Exploitation Phase:} Each player $i$ plays arm $k_i^*(l)$ for $2^{l}$ time slots;
\STATE $t=t+ MN \gamma+2^{l}$, $l=l+1$;
\ENDWHILE
\end{algorithmic}
\end{algorithm}

\subsection{Regret analysis}
In both these algorithms, the total regret can be thought as the sum of three different regret terms. The time slots spent in exploration are considered to contribute to regret as the first term, $\tilde{\sR}^{O}(T)$. At the end of every exploration phase, a bipartite matching algorithm is run and each run adds cost $C$ to the second term of regret $\tilde{\sR}^{C}(T)$.  The cost $C$ depends on two parameters: (a) the precision of the bipartite matching algorithm $\epsilon_{1}>0$, and (b) the precision of the index representation $\epsilon_{2}>0$. A bipartite matching algorithm has an $\epsilon_{1}$-precision if it gives an $\epsilon_{1}$-optimal matching. This would happen, for example, when such an algorithm is run only for a finite number of rounds. The index has an $\epsilon_{2}$-precision if any two indices are not distinguishable if they are closer than $\epsilon_{2}$. This can happen, for instance, when indices must be communicated to other players with a finite number of bits. Thus, the cost $C$ is a function of $\epsilon_{1}$ and $\epsilon_{2}$, and can be denoted as $C(\epsilon_{1}, \epsilon_{2})$, with $C(\epsilon_{1}, \epsilon_{2}) \rightarrow \infty$ as $\epsilon_{1}$ or  $\epsilon_{2} \rightarrow 0$. Since, $\epsilon_{1}$ and $\epsilon_{2}$ are the parameters that are fixed \textit{a priori}, we consider $\epsilon=\min( \epsilon_{1},\epsilon_{2})$ to specify  both precisions. We shall denote this computation and communication cost by $C(\epsilon)$ as different communication methods and implementations of distributed bipartitie matching will give different costs.

The third term in the regret expression, $\tilde{\sR}^{I}(T)$, comes from non-optimal matchings in the exploitation phase, i.e.,  if the matching $\mathbf{k}^{*}(l)$ is not the optimal matching $\mathbf{k}^{**}$. Thus, we have the total expected regret of the $\tt dE^3$ and $\tt dE^3$-$\tt TS$ policies to be given by,
\begin{eqnarray}
\label{eq:dregret-as-sum-of3}
\tilde{\sR}(T) = \tilde{\sR}^{O}(T) + \tilde{\sR}^{I}(T) + \tilde{\sR}^{C}(T).
\end{eqnarray}
We now give the main results of this section. 
%performance bounds for $\tt dE^3$ and $\tt dE^3$-$\tt TS$ policies in the main result of this section.
% As before, $\tilde{\sR}$ will be used to denote regret for either policy when both are within context. 
\begin{theorem}
\label{thm:dE^3-reg} 
(i) Let $\epsilon > 0$ be the precision of the bipartite matching algorithm and the precision of the index representation. If $\Delta_{\min}$ is known, choose $\epsilon$ such that $0 < \epsilon < \D_{\min}/(M+1)$, set $\gamma = \left \lceil 2 M^{2}/(\Delta_{\min} - (M+1)\epsilon)^{2} \right \rceil$ and $\gamma_\beta = \left \lceil 8 M^{2}/(\Delta_{\min} - (M+1)\epsilon)^{2} \right \rceil$. Then, the expected regrets of the ${\tt dE^3}$ and $\tt dE^3$-$\tt TS$ policies are,
\begin{align}
&\tilde{\sR}_{\tt dE^3}(T) \leq  M N \Delta_{\max} \gamma \log(T)  + M N  ~ C (\epsilon) \log(T)  \nonumber  \\
&\hspace{2cm} + 8 M N \Delta_{\max}, \\
&\tilde{\sR}_{\tt E^3-\tt TS}(T) \leq  M N \Delta_{\max} \gamma_\beta \log(T)  + M N ~ C (\epsilon) \log(T)  \nonumber \\
&\hspace{2cm}+ 16 M N \Delta_{\max}.
\end{align}  
Note that, in the above expressions, $\epsilon$ is a chosen constant. Thus, $\tilde{\sR}(T) = O(\log(T))$ for both policies. 

\noindent (ii)  If $\Delta_{\min}$ is not known, choose $\gamma=\gamma_{t}$, where $\{\gamma_{t}\}$ is  a positive sequence such that $\gamma_{t} \rightarrow \infty$ as $t \rightarrow \infty$. Also choose $\epsilon=\epsilon_{t}$,   where $\{\epsilon_{t}\}$ is  a positive sequence  such that $\epsilon_{t} \rightarrow 0$ as $t \rightarrow \infty$. Then,  
\begin{align}
\tilde{\sR}_{\tt d}(T) &\leq M N \Delta_{\max} \gamma_{T} \log(T)  + M N  C (\epsilon_{T}) \log(T)   \nonumber  \\
&\hspace{2cm} + M N B
\end{align}   
where $\tilde{\sR}_{\tt d}(T) =  \max \{\tilde{\sR}_{\tt dE^3}(T), \tilde{\sR}_{\tt dE^3-TS}(T)\}$ and $B$ is a constant independent of $T$. In particular, for $C(\epsilon) =  \epsilon^{-1}$, choose $\gamma_{t} = \log^{\delta}t, \epsilon_{t} = \log^{-\delta}t, \delta \in (0,1)$,  and we get 
\begin{align}
&\tilde{\sR}_{\tt d}(T) \leq M N \Delta_{\max} \log^{(1+\delta)}(T)  + M N \log^{(1+\delta)}(T) \nonumber  \\
&\hspace{2cm}+ M N B(\delta)
\end{align} 
where $B(\delta) = b_{0}2^{l(\delta)}, l(\delta) = (\Delta^{2}_{\min}/4)^{-1/\delta} $ and $b_{0}$ is a constant independent of $\delta$. 

\end{theorem}
Proof is given in Appendix~\ref{app:proof-2}.

\subsection{Distributed Bipartite Matching}
\label{sec:dbm} 

Both  $\tt dE^3$ algorithm and  $\tt dE^3$-$\tt TS$ algorithm use the distributed bipartite matching algorithm as a subroutine. In Section \ref{subsec:mp-policy} we have given an abstract description of this distributed bipartite matching algorithm.  We now present one such algorithm, namely, Bertsekas' auction algorithm \cite{Be92}, and its distributed implementation. We note that the presented algorithm is not the only one that can be used. Both  $\tt dE^3$ algorithm and  $\tt dE^3$-$\tt TS$ algorithm  will work with a distributed implementation of any bipartite matching algorithm, e.g. algorithms given in \cite{ZaSpPa08}.  

Consider a bipartite graph with $M$ players on one side, and $N$ arms on the other, and $M\leq N$. Each player $i$ has a value $\mu_{i,j}$ for each arm $j$. Each player knows only his own values. Let us denote by $k^{**}$, a matching that maximizes the matching surplus $\sum_{i,j} \mu_{i,j}x_{i,j}$, where the variable $x_{i,j}$ is 1 if $i$ is matched with $j$, and 0 otherwise. Note that $\sum_i x_{i,j} \leq 1, \forall j$, and $\sum_j x_{i,j} \leq 1, \forall i$. Our goal is to find an $\epsilon$-optimal matching. We call any matching $k^*$ to be $\epsilon$-optimal if $\sum_{i} \mu_{i,k^{**}(i)} - \sum_{i} \mu_{i,k^*(i)} \leq \epsilon$.

\begin{algorithm}{}
\caption{: ${\tt dBM_{\e}}$ ( Bertsekas Auction Algorithm)}
\label{algo:dBM}
{\small \begin{algorithmic}[1]
\STATE All players $i$ initialize prices $p_j = 0, \forall ~\text{channels} ~j$;
\WHILE{(prices change)}
\STATE Player $i$ communicates his preferred arm $j_i^*$ and bid $b_i = \max_j (\mu_{ij}-p_j) - \text{second.max}_j(\mu_{ij}-p_j) + \frac{\epsilon}{M}$ to all other players.
\STATE Each player determines on his own if he is the \textit{winner} $i_j^*$ on arm $j$;
\STATE All players set prices $p_j = \mu_{i_j^*,j}$;
\ENDWHILE
\end{algorithmic}
}
\end{algorithm}

Here, $\text{second.max}_j$ is the second highest maximum over all $j$. The \textit{best} arm for a player $i$ is arm $j_i^* = \arg \max_j (\mu_{i,j}-p_j)$. The \textit{winner} $i_j^*$ on an arm $j$ is the one with the highest bid.  

The following lemma in \cite{Be92} establishes that Bertsekas' auction algorithm will find the $\e$-optimal matching in a finite number of steps. 
%with an upper bound that depends on problem primitives.

\begin{lemma}\cite{Be92}
Given $\epsilon>0$, Algorithm~\ref{algo:dBM} with rewards $\mu_{i,j}$, for player $i$ playing the $j$th arm, converges to a matching $k^*$ such that $\sum_{i} \mu_{i, k^{**}(i)} - \sum_{i} \mu_{i, k^*(i)} \leq \epsilon$ where $k^{**}$ is an optimal matching. Furthermore, this convergence occurs in less than $(M^2\max_{i,j}\{\mu_{i,j}\})/\epsilon$ iterations.
\end{lemma}

 Our only assumption here is going to be that each user can observe a channel, and determine if there was a successful transmission on it, a collision, or no transmission, in a given time slot. This consists of $J$ rounds. In each round, users transmit in a round robin fashion, where she can signal her channel preferences  using $\lceil{\log M} \rceil $ bits and bid values (difference of top two indices) using $\lceil{\log 1/\e_1}\rceil$ bits.  The number of rounds $J$ is  chosen so that the ${\tt dBM}$ algorithm (based on Algorithm \ref{algo:dBM}) returns an $\e_2$-optimal matching.  More details on this implementation is given in \cite{KaNaJa14}.

\section{Simulations}
\label{sec:simulation}

We conducted extensive simulations comparing the performances of the proposed policies with prior work. The results are presented in the respective sections below.

\subsection{Single player bandit policies}
\label{subsec:sp-simulations}

For the single player setting, we considered a four-armed bandit problem with rewards for arms drawn independently from Bernoulli distributions with means $0.1,~0.5,~0.6,~0.9$. The scenario was simulated over a fixed time horizon $T=2,000,000$ timeslots and the performance of the proposed single-player policies was evaluated. The performance of each policy was averaged over 10 sample runs and the results presented here. Different true means and distributions were also considered and they gave similar rankings for the algorithms. In the interest of space, those scenarios are not presented.

In Figure~\ref{fig:simulationsE3UCB4}, the single player policies proposed in this paper, $\tt E^3$ and $\tt E^3$-$\tt TS$, are compared with the benchmark $\tt UCB_1$ policy. $\Delta_{\min}$ is assumed to be known $(0.1)$ and, consequently, $\gamma$ is fixed. The bound for $\tt E^3$-$\tt TS$ is also shown with the dashed line. It can be observed that, although, all three policies have logarithmic order of regret performance in time, the new $\tt E^3$ and $\tt E^3$-$\tt TS$ policies perform slightly worse than the $\tt UCB_1$ policy. This is attributable to the deterministic exploration phase length which must take into account the worst-case scenario. However, as we shall see in the next section, this gives us a significant performance advantage in the multiplayer setting.

\begin{figure}[tbh]
\hspace{-0.3in}
\centering
\includegraphics[width=9cm, height=5cm]{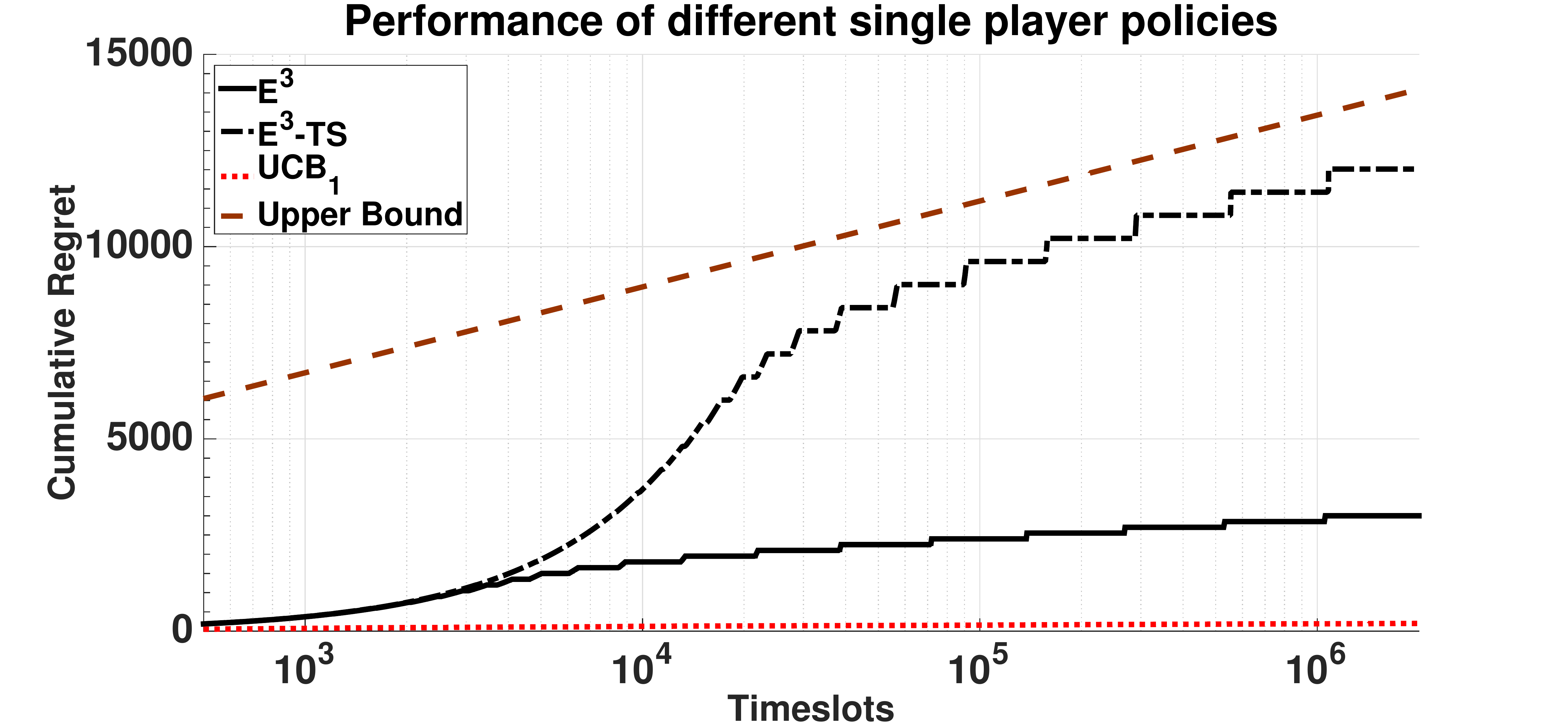} 
\caption{Figure showing growth of cumulative regret of the $\tt E^3$, $\tt E^3$-$\tt TS$ and $\tt UCB_1$ algorithms for a four-armed single-player bandit problem with true means $[0.1,  0.5, 0.6, 0.9]$ (no computation cost), with time plotted on log scale.} 
\label{fig:simulationsE3UCB4}
\end{figure}

Note that in Figure~\ref{fig:simulationsE3UCB4}, computation cost is assumed to be zero. If computation cost were included, $\tt E^3$ and $\tt E^3$-$\tt TS$ would retain their logarithmic regret performance. However, the cumulative regret of $\tt UCB_1$ would grow linearly, just as with $\tt TS$~\cite{KaNaJa14}.

\subsection{Multiplayer bandit policies}
\label{subsec:mp-simulation}

We now present the empirical performance of the proposed $\tt dE^3$ and $\tt dE^3$-$\tt TS$ policies. We consider a three-player, three-armed bandit setting. Rewards for each arm are generated independently from a Bernoulli distribution with means $0.2,~0.25,~0.3$ for player 1, $0.4,~0.6,~0.5$ for player 2 and $0.7,~0.9,~0.8$ for player 3. A time horizon spanning 20 epochs  was considered. $\epsilon=0.001$ was used as the tolerance for the bipartite matching algorithm, which was done using $\tt dBM_\epsilon$, a distributed implementation of Bertsekas' auction algorithm. The performance of each policy was averaged over 10 sample runs. $\gamma$ was set equal to $100$ for $\tt dE^3$ and $400$ for $\tt dE^3$-$\tt TS$ (see analysis for the reason for differing $\gamma$'s). A fixed per unit cost each time the distributed bipartite matching algorithm $\tt dBM$ is run, is included in the setting to model communication cost in the decentralized setting.

The plot of the growth of cumulative regret with time of $\tt dE^3$, $\tt dE^3$-$\tt TS$ and $\tt dUCB_4$ is shown in Figure~\ref{fig:simulationsE3E3BDUCB4}. We can see that the logarithmic regret performance of $\tt dE^3$ and $\tt dE^3$-$\tt TS$ clearly outperforms the $\log^2 T$-regret performance of our earlier $\tt dUCB_4$ policy~\cite{KaNaJa14}. The dashed line curve is the theoretical upper bound on the performance of $\tt dE^3$-$\tt TS$.

\begin{figure}[tbh]
\hspace{-0.3in}
\includegraphics[width=10cm, height=5cm]{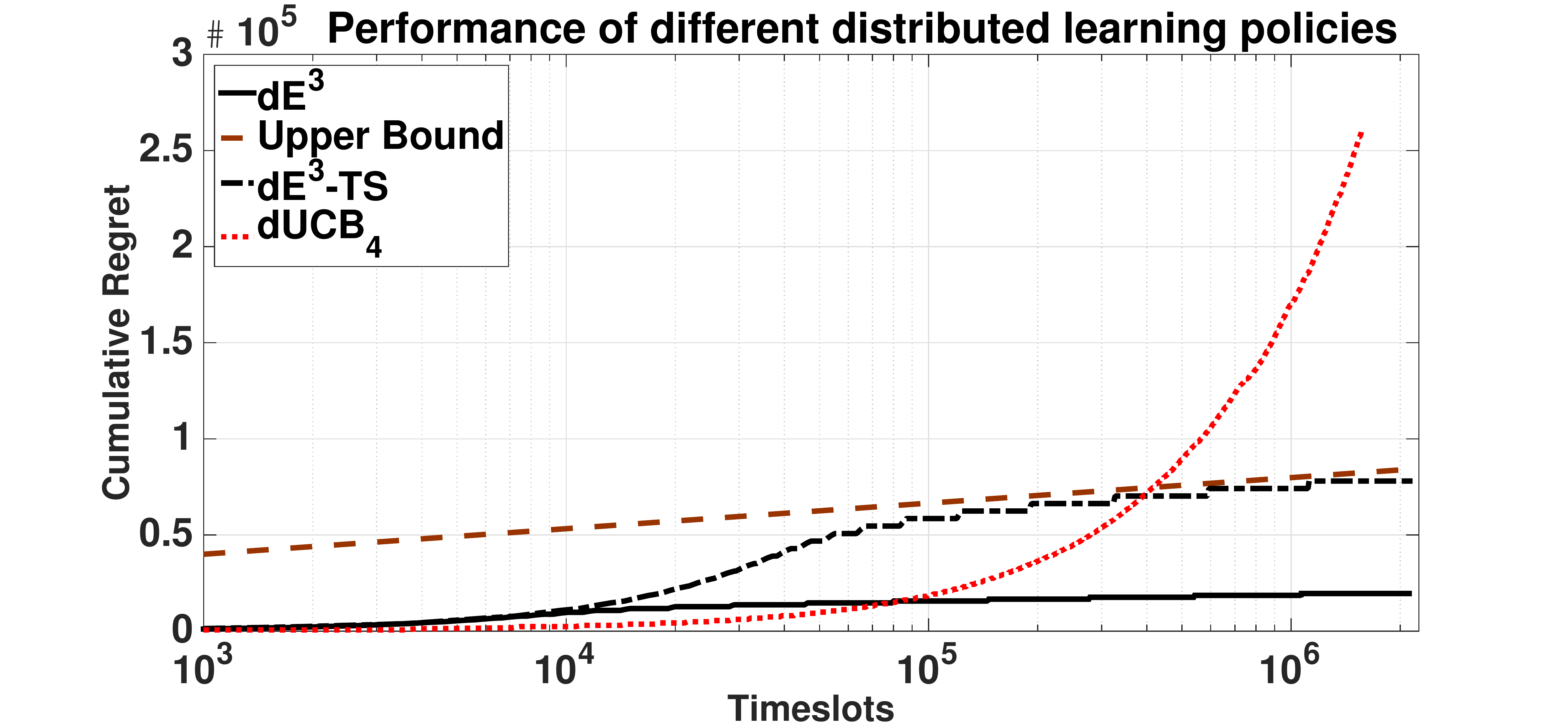}
\caption{Figure showing growth of cumulative regret of the $\tt dE^3$ and $\tt dE^3$-$\tt TS$ algorithms for a three-player, three-armed bandit setting with true means $[0.2 0.25 0.3; 0.4 0.6 0.5; 0.7 0.9 0.8]$ (communication cost included), with time plotted on log scale.} 
\label{fig:simulationsE3E3BDUCB4}
\end{figure}

\section{Conclusion}
\label{sec:conclusion}

We designed two closely related single player and multiplayer bandit policies that achieve logarithmic or near-logarithmic regret performance depending on the assumptions of the model. Both policies have deterministic exploration and exploitation phases, which make them well-suited to decentralization for use in the multiplayer setting.

Performances of these policies were compared to prior work in the literature. They were shown to outperform previous policies for multiplayer bandits, but not for the single player model due to the deterministic phases of these new policies. While we have approached logarithmic regret performance under certain assumptions in the multiplayer model, the question of whether a policy under truly general conditions can achieve fully logarithmic regret remains open.

\bibliographystyle{IEEEtran}
\bibliography{References_TS-MAB}

\appendices

\section{Proof of Theorem 1}
\label{app:proof-1} 

In the following proof, the subscript $\beta$ for $\gamma_\beta$ will be omitted where the context refers to both $\tt E^3$ and $\tt E^3$-$\tt TS$ together. Also, $\sR$ will be used to denote regret for either policy when both are within context.
\vspace{-0.5cm}
\subsection{$\Delta_{\min}$ is known:}
\label{app:proof-1-SubsecA}

We denote the expected regrets incurred in the exploration phases with $\tilde{\sR}^{O}(T)$, exploitation phases with $\tilde{\sR}^{I}(T)$ and due to computation with $\tilde{\sR}^{C}(T)$. Then,
\begin{eqnarray}
\label{eq:regret-as-sum-of3}
\tilde{\sR}(T) = \tilde{\sR}^{O}(T) +  \tilde{\sR}^{C}(T) + \tilde{\sR}^{I}(T), 
\end{eqnarray}
for both $\tt E^3$ and $\tt E^3$-$\tt TS$ policies.

Let $T$ be in the $l_{0}$-th exploitation epoch. By construction, $T \geq N \gamma l_{0} + 2^{l_{0}} - 2$. Thus, $\log T \geq l_{0}$ and,

\begin{equation}
\label{eq:reg-O-E^3}
\tilde{\sR}_{\tt}^{O}(T) = \gamma l_{0} \sum^{N}_{j=2} \Delta_{j} \leq  N \gamma l_{0} \Delta_{\max} \leq  N \gamma \log(T) \Delta_{\max},
\end{equation}
where $\Delta_j = \mu_{1} - \mu_{j}$. Also, using the definition of computation cost,
\begin{equation}
\label{eq:reg-C-E^3}
\tilde{\sR}^{C}(T) =  N C l_{0} \leq   N C \log(T). 
\end{equation}

Now, $\tilde{\sR}^{I}(T) = \mathbb{E} \left[ \sum^{N}_{j=2} \Delta_{j} \tilde{n}_{j}(T) \right]$ where $\tilde{n}_{j}(T)$ is the number of times arm $j$ has been played during the exploitation phases. For $\tt E^3$,
\begin{align}
\tilde{n}_j(T) &=  \sum^{l_{0}}_{l=1} 2^{l} \bb1 \{\overline{X}_{j}(l) > \max_{1\leq i \leq N} \overline{X}_i(l) \} \nonumber \\
&\leq \sum^{l_{0}}_{l=1} 2^{l} \bb1 \{ \overline{X}_{j}(l) > \overline{X}_{1}(l) \}
\end{align}
Similarly, for $\tt E^3$-$\tt TS$, $\tilde{n}_{j}(T) \leq \sum^{l_{0}}_{l=1} 2^{l} \bb1 \{\theta_{j}(l) > \theta_{1}(l) \}$. Thus,
\begin{align}
\label{eq:reg-I-E^3-s1 }
 \tilde{\sR}_{\tt E^3}^{I}(T) & \leq \Delta_{\max} \sum_{l=1}^{l_{0}} \sum^{N}_{j=2}  2^{l} \bbP( \overline{X}_{1}(l) < \overline{X}_{j}(l)),
\end{align}
\begin{align}
\label{eq:reg-I-E^3-b-s1 }
\text{and,}~  \tilde{\sR}_{\tt E^3-\tt TS}^{I}(T) & \leq \Delta_{\max} \sum_{l=1}^{l_{0}} \sum^{N}_{j=2}  2^{l} \bbP( \theta_{1}(l) < \theta_{j}(l)).
\end{align}

The following two lemmas bound the event probabilities above for the $\tt E^3$ and $\tt E^3$-$\tt TS$ policies.

\begin{lemma}
\label{lem:E^3-prob}
For $\tt E^3$, with $\gamma = \lceil \frac{2}{\Delta_{\min}^2} \rceil$,
\begin{equation}
\bbP (\overline{X}_1(l) < \overline{X}_j(l)) \leq 2e^{-l},~\forall j>1.
\end{equation}

\end{lemma}

\begin{proof}

The event $\{\overline{X}_{1}(l) < \overline{X}_{j}(l) \}$ implies at least one of the following events:
\begin{align*}
%\label{eq:events-E^3}  
A_{j}:=&\{ \overline{X}_{j}(l)  - \mu_{j} > \Delta_{j}/2\}, B_{j}:=&\{ \overline{X}_{1}(l)  - \mu_{1} < - \Delta_{j}/2\}.
\end{align*}
%The above follows from the fact that three events A,B and C are such that $A \subset ( B \cup C )$, if and only if, $A \cap \overline{B} \cap \overline{C} = \emptyset$. Substituting $A = \{\overline{X}_{1}(l) < \overline{X}_{j}(l) \}$, $B = \{ \overline{X}_{j}(l)  - \mu_{j} > \Delta_{j}/2\}$, and $C = \{ \overline{X}_{1}(l)  - \mu_{1} < - \Delta_{j}/2\}$ in the expression, the result follows. 
Using the Chernoff-Hoeffding bound and choosing $\gamma = \lceil \frac{2}{\Delta_{min}^{2}} \rceil$, we get,
\begin{flalign}
\label{eq:eventprob-E^3}
\bbP(A_{j})  \leq e^{-2 l \gamma \Delta^{2}_{j}/4} \leq e^{-l}, ~\text{and similarly,}~ \mathbb{P}(B_{j}) \leq e^{-l}.
\end{flalign} 
By the union bound, we get $\bbP( \overline{X}_{1}(l) < \overline{X}_{j}(l)) \leq 2 e^{-l}$.

\end{proof}

\begin{lemma}
\label{lem:E^3-b-prob}
For $\tt E^3$-$\tt TS$, with $\gamma_\beta = \lceil \frac{8}{\Delta_{\min}^2} \rceil$,
\begin{equation}
\bbP (\theta_1(l) < \theta_j(l)) \leq 4e^{-l},~\forall j>1. 
\end{equation}

\end{lemma}

\begin{proof}

Without loss of generality, we will assume the underlying reward distributions of the arms to have a Bernoulli distribution to simplify the analysis. This eliminates the need for line 5 in the $\tt E^3$-$\tt TS$ policy illustrated in Algorithm~\ref{algo:E^3-b}. However, this assumption can be relaxed without any change to the results.

As in Lemma~\ref{lem:E^3-prob}, the event $\{\theta_{1}(l) < \theta_{j}(l) \}$ implies at least one of the events:
\begin{equation}
\label{eq:events-E^3-b}
A_{j}:=\{ \theta_{j}(l)  - \mu_{j} > \Delta_{j}/2\}, B_{j}:=\{ \theta_{1}(l)  - \mu_{1} < - \Delta_{j}/2\}.
\end{equation}

Let $m_j(l)$ denote the number of plays of arm $j$ during the exploration phases after the $l$-th exploration epoch, and let $s_j(l)$ be the number of successes ($r = 1$) in these plays. Then, $\theta_j(l)$ is sampled from a $\beta(s_j(l) + 1, m_j(l) - s_j(l) + 1)$ distribution.

Additionally, let $A(l)$ denote the event $\{\frac{s_j(l)}{m_j(l)} < \mu_j + \frac{\Delta_j}{4}\}$. Then,

\begin{equation}
\label{eq:E^3-b-term}
\bbP (\theta_j(l) \geq \mu_j + \frac{\Delta_j}{2}) \leq \bbP (\overline{A}(l)) + \bbP (\theta_j(l) \geq \mu_j + \frac{\Delta_j}{2},A(l)).
\end{equation}

The first term in the expression,
\begin{align*}
\bbP (\overline{A}(l)) &= \bbP \Big(\frac{s_j(l)}{m_j(l)} \geq \mu_j + \frac{\Delta_j}{4}\Big) \leq \exp\Big( \frac{-2\gamma_{\beta} l\Delta_j^2}{16} \Big),
\end{align*}
where the last inequality comes from the Chernoff-Hoeffding inequality and by noting that $\frac{s_j(l)}{m_j(l)}$ is a random variable with mean $\mu_j$. Also, $m_j(l) = \gamma_\beta l$.

The second term,
\begin{align}
\label{eq:E^3-b-term-1}
&\bbP (\theta_j(l) \geq \mu_j + \frac{\Delta_j}{2},A(l)) \nonumber\\
&= \bbP \Big(\theta_j(l) \geq \mu_j + \frac{\Delta}{2},\frac{s_j(l)}{m_j(l)} < \mu_j + \frac{\Delta_j}{4} \Big) \nonumber \\
&\leq \bbP \Big( \theta_j(l) > \frac{s_j(l)}{m_j(l)} + \frac{\Delta_j}{4} \Big) \nonumber \\
&= \bbP \Big( \beta(s_j(l) + 1, m_j(l) - s_j(l) + 1) > \frac{s_j(l)}{m_j(l)} + \frac{\Delta_j}{4} \Big) \nonumber \\
&= \bbE \Big[ F^B_{m_j(l) + 1, \frac{s_j(l)}{m_j(l)} + \frac{\Delta_j}{4}}(s_j(l)) \Big] \nonumber \\
&\leq \bbE \Big[ F^B_{m_j(l), \frac{s_j(l)}{m_j(l)} + \frac{\Delta_j}{4}}(s_j(l)) \Big].
\end{align}
Here, $F^B_{n,p}(x)$ is the cdf of the $\tt binomial(n,p)$ distribution. The equality in the second-to-last line comes from the fact that $F^{\beta}_{a,b}(x) = 1 - F^B_{a + b - 1, x}(a - 1)$, where $F^{\beta}_{a,b}(x)$ is the cdf of the $\tt \beta(a,b)$ distribution~\cite{AgGo12}. The inequality on the last line is a standard inequality for binomial distributions.

But, by the Chernoff-Hoeffding inequality, it can be seen that $F^B_{n,p}(np - n\delta) \leq \exp (-2n\delta^2)$. Thus,

\begin{equation}
\label{eq:E^3-b-term-2}
\bbP (\theta_j(l) \geq \mu_j + \frac{\Delta}{2},A(l)) \leq \exp \Big( \frac{-2\gamma_\beta l \Delta_j^2}{16} \Big)
\end{equation}

Setting $\gamma_\beta := \lceil \frac{8}{\Delta_{\min}^2} \rceil$ in~\eqref{eq:E^3-b-term-1} and~\eqref{eq:E^3-b-term-2}, we get,

\begin{equation}
\bbP (\theta_j(l) \geq \mu_j + \frac{\Delta_j}{2}) \leq 2e^{-l}.
\end{equation}

Similarly, $\bbP (\theta_1(l) \leq \mu_1 - \frac{\Delta_j}{2}) \leq 2e^{-l}$, and the claim of the lemma follows from the union bound.

\end{proof}

Continuing with the proof of Theorem 1, thus, 
\begin{align}
\label{eq:reg-I-E^3-s2}
\tilde{\sR}_{\tt E^3}^{I}(T) & \leq  \Delta_{\max} \sum_{l=1}^{l_{0}} \sum^{N}_{j=2}  2^{l}  2 e^{-l} \leq 2 N \Delta_{\max} \sum_{l=0}^{\infty}  \left( 2/e\right)^{l} \nonumber \\
& \leq 2 N \Delta_{\max} / (1-(2/e)) < 8 N  \Delta_{\max}, 
\end{align}
\begin{align}
\label{eq:reg-I-E^3-b-s2}
\tilde{\sR}_{\tt E^3-\tt TS}^{I}(T) & \leq  \Delta_{\max} \sum_{l=1}^{l_{0}} \sum^{N}_{j=2}  2^{l}  4 e^{-l} \leq 4 N \Delta_{\max} \sum_{l=0}^{\infty}  \left( 2/e\right)^{l} \nonumber \\
& \leq 4 N \Delta_{\max} / (1-(2/e)) < 16 N  \Delta_{\max}, 
\end{align}

Now, combining all the terms, we get
\begin{equation}
\label{eq:reg-total-E^3}
\tilde{\sR}_{\tt E^3}(T) \leq N \gamma \log T \Delta_{\max} + C \log(T) + 8 N \Delta_{\max}.
\end{equation} 
\begin{equation}
\label{eq:reg-total-E^3-b}
\tilde{\sR}_{\tt E^3-\tt TS}(T) \leq N \gamma_\beta \log T \Delta_{\max} + C \log(T) + 16 N \Delta_{\max}.
\end{equation}

\subsection{$\Delta_{\min}$ is unknown:} 
\label{app:proof-1-SubsecA}

Suppose $t_{l}$ be the time $t$ at which $l$th exploration phase begins. For the clarity of explanation, we assume that $\gamma$ changes only in the beginning of an exploration phase. So, in the $l$th exploration phase, each arms is played $\gamma_{t_{l}}$ times in a round robin manner.  

As in the proof given in the previous subsection, let $T$ be in the $l_{0}$th exploitation epoch. By construction, $T \geq N \sum^{l_{0}}_{l=1} \gamma_{t_{l}} + 2^{l_{0}} - 2$. Thus, $\log T \geq l_{0}$ and,
\begin{equation}
\label{eq:reg-O-E^3-B}
\tilde{\sR}_{\tt}^{O}(T) = \sum^{l_{0}}_{l=1} \gamma_{t_{l}} \sum^{N}_{j=2} \Delta_{j} \leq  N \Delta_{\max}  \gamma_{T} l_{0} \leq  N \Delta_{\max} \gamma_{T} \log(T) .
\end{equation}
The second inequality is from the fact that $\gamma_{t}$ is a monotone increasing sequence. 

The computation cost is same as before, i.e., 
\[ \tilde{\sR}^{C}(T) =  N C l_{0} \leq   N C \log(T).  \] 
\begin{align*}
\text{Using \eqref{eq:eventprob-E^3},}~ \tilde{\sR}_{\tt E^3}^{I}(T) & \leq \Delta_{\max} \sum_{l=1}^{l_{0}} \sum^{N}_{j=2}  2^{l} \bbP( \overline{X}_{1}(l) < \overline{X}_{j}(l)) \\
 & \leq N \Delta_{\max} \sum_{l=1}^{\infty}  2^{l} e^{-b_{1} \sum^{l}_{k=1} \gamma_{t_{k}} } 
\end{align*}
where $b_{1} = \Delta^{2}_{\min}/2$. Since $ \gamma_{t} \rightarrow \infty$  monotonically (and $\gamma_{t_{k}} \geq 1$), there exists an $l^{\prime}$ such that $b_{1} \sum^{l}_{k=1} \gamma_{t_{k}} \geq l,  \forall l > l^{\prime}$. Then, 
\begin{align*}
\tilde{\sR}_{\tt E^3}^{I}(T) & \leq N \Delta_{\max} \left(  \sum_{l=1}^{l'-1}  2^{l} e^{-b_{1} \sum^{l}_{k=1} \gamma_{t_{k}} }  + \sum_{l=l'}^{\infty} (2/e)^{l}  \right) \\
&\leq N \Delta_{\max} B
\end{align*}
where $B$ is a finite constant, independent of $T$. 

When $\gamma_{t_{l}} =  \log^{\delta}t_{l}$ for $\delta \in (0,1)$, it is easy to see that $\gamma_{t_{l}} \geq l^{\delta}, \forall l$. Then, $ \sum^{l}_{k=1} \gamma_{t_{k}} \geq  \sum^{l}_{k=1} k^{\delta} \geq \int^{l+1}_{x=1} (x-1)^{\delta} dx \geq 0.5 l^{(1+\delta)}$.  From this, $l' = (2/b_{1})^{1/\delta}$. Then, we can get $B = B(\delta) = 2^{l'}$.

\section{Proof of Theorem 2}
\label{app:proof-2}

We first show that if $\Delta_{\min}$ is known, we can choose an $\epsilon < \Delta_{\min}/(M+1)$, such that ${\tt dE^3}$ and $\tt dE^3$-$\tt TS$ algorithms achieve a logarithmic regret growth with $T$. If $\Delta_{\min}$ is not known, we can pick a positive monotone sequence $\{\epsilon_{t}\}$  such that $\e_t \to 0$, as $t \to \infty$. In a decentralized bipartite matching algorithm, the precision $\e$ will depend on the amount of information exchanged.

The proof will be illustrated here only for the $\tt dE^3$ policy since the differences between it and the analysis of the $\tt dE^3$-$\tt TS$ policy are similar to those found in Theorem~\ref{thm:E^3-regret}.

Let us denote the optimal bipartite matching with $\mathbf{k}^{**} \in \sP(N)$  such that $\mathbf{k}^{**} \in \arg \max_{\mathbf{k} \in \sP(N) } \sum_{i=1}^{M} \mu_{i,\mathbf{k}_{i}}$. Denote $\mu^{**} := \sum_{i=1}^{M} \mu_{i, \mathbf{k}_i^{**}}$, and define $\Delta_{\mathbf{k}} :=  \mu^{**} - \sum_{i=1}^{M} \mu_{i,\mathbf{k}_{i} },~ \mathbf{k} \in \sP(N)$.

Let $\Delta_{\min}  = \min_{\mathbf{k} \in \sP(N), \mathbf{k} \neq \mathbf{k}^{**}} \Delta_{\mathbf{k}}$ and \\$\Delta_{\max}  = \max_{\mathbf{k} \in \sP(N)} \Delta_{\mathbf{k}}$.  We assume  $\Delta_{\min}>0$.

\subsection{$\Delta_{\min}$ is known:}  

Let $T$ be in the $l_{0}$th exploitation epoch. It follows that, $T \geq M N \gamma l_{0} + 2^{l_{0}} - 2$ and, hence, $\log T \geq l_{0}$. Then, 
\begin{equation}
\label{eq:reg-O-dE^3}
\tilde{\sR}_{\tt dE^3}^{O}(T) = M N \gamma l_{0} \Delta_{\max} \leq  M N \gamma \log(T) \Delta_{\max}.
\end{equation}
\text{Also, by definition,}
\begin{equation}
\label{eq:reg-C-dE^3}
\tilde{\sR}_{\tt dE^3}^{C}(T) =  M N C(\epsilon) l_{0} \leq M N  C(\epsilon) \log(T). 
\end{equation}

A suboptimal matching occurs in the $l$-th exploitation epoch if the event $\{\sum^{M}_{i=1}  \overline{X}_{i,k^{**}_{i}}(l)   < (M+1) \epsilon + \sum^{M}_{i=1} \overline{X}_{i,k^{*}_{i}}(l)\}$ occurs. If each index has an error of at most $\epsilon$, the sum of M terms
may introduce an error of at most $M \epsilon$. In addition, the distributed bipartite matching algorithm ${\tt dBM_{\epsilon} }$ itself yields only an $\epsilon$-optimal matching. This accounts for the term $(M+1) \epsilon$ above.

Clearly, as in the single player case, 
\begin{align}
\tilde{\sR}_{\tt dE^3}^{I}(T) &\leq \Delta_{\max} \sum^{l_{0}}_{l=1} 2^{l} \mathbb{P}\Big( \sum^{M}_{i=1} \overline{X}_{i,k^{**}_{i}}(l) < (M+1) \epsilon \nonumber \\
& + \sum^{M}_{i=1} \overline{X}_{i,k^{*}_{i}}(l) \Big).
\end{align}
The event $\left \{ \sum^{M}_{i=1} \overline{X}_{i,k^{**}_{i}}(l)   < (M+1) \epsilon + \sum^{M}_{i=1} \overline{X}_{i,k^{*}_{i}}(l) \right \}$ implies at least one of the following events
\begin{equation}
\label{eq:events-dE^3}
A_{i,j}:=\{ |\overline{X}_{i,j}(l)  - \mu_{i,j}| > (\Delta_{\min}-(M+1)\epsilon)/2M\},
\end{equation}
for $1 \leq i \leq M, 1 \leq j \leq N$. By the Chernoff-Hoeffding bound, and then using the fact that $\gamma = \lceil  2 M^{2}/(\Delta_{\min}-(M+1)\epsilon)^{2} \rceil$, 
\begin{equation}
\label{eq:eventprob-dE^3}
\bbP(A_{i,j})  \leq 2 e^{-2 l \gamma (\Delta_{\min}-(M+1)\epsilon)^{2}/4M^{2}} \leq 2 e^{-l}.
\end{equation}
Then, by using the union bound, 
\begin{align}
\tilde{\sR}_{\tt E^3}^{I}(T)  & \leq \Delta_{\max} \sum^{l_{0}}_{l=1} 2^{l} \sum^{M}_{i=1} \sum^{N}_{j=1}  \mathbb{P}(A_{i,j}) \nonumber \\
& \leq \Delta_{\max} 2 M N   \sum_{l=0}^{\infty  }  \left( 2/e\right)^{l} = 2 \Delta_{\max} M N/ (1-(2/e)) \nonumber \\
& < 8 M N \Delta_{\max}.
\end{align}
Combining all the terms, we get
\begin{align}
\label{eq:reg-total-dE^3}
\tilde{\sR}_{\tt dE^3}(T) \leq & M N \gamma \log(T) \Delta_{\max} + M N C (\epsilon) \log(T) \nonumber \\
& + 8 M N \Delta_{\max}.
\end{align} 

In a similar manner,
\begin{align}
\label{eq:reg-total-dE^3-b}
\tilde{\sR}_{\tt E^3-\tt TS}(T) \leq & M N \gamma_\beta \log(T) \Delta_{\max} + M N C (\epsilon) \log(T) \nonumber \\
& + 16 M N \Delta_{\max},
\end{align} 
where $\gamma_\beta = \lceil  8 M^{2}/(\Delta_{\min}-(M+1)\epsilon)^{2} \rceil$.

\subsection{$\Delta_{\min}$ is unknown:} 

The proof is similar to the proof of the analogous case of Theorem \ref{thm:E^3-regret}, and is omitted.

\begin{IEEEbiography}[{\includegraphics[width=1in,height
=1.25in,clip,keepaspectratio]{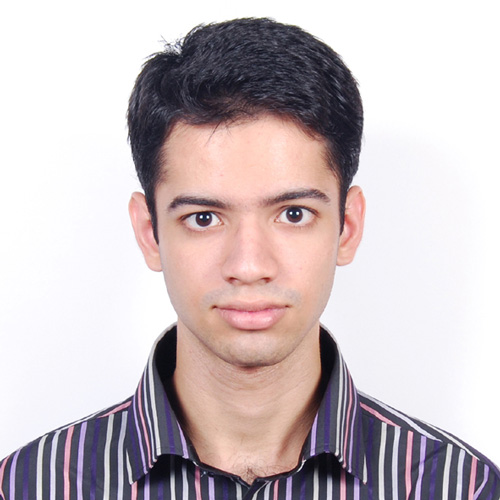}}]{Naumaan Nayyar} received a PhD in Electrical Engineering from the University of Southern California in 2015. He also holds an MS in Statistics from USC and a B.Tech. in EE from the Indian Institute of Technology, Bombay. He is currently in a healthcare startup after working for a while at IBM Research. His research interests are in the areas of learning, online optimization, statistical inference and stochastic control theory.
\end{IEEEbiography}

\begin{IEEEbiography}[{\includegraphics[width=1in,height
=1.25in,clip,keepaspectratio]{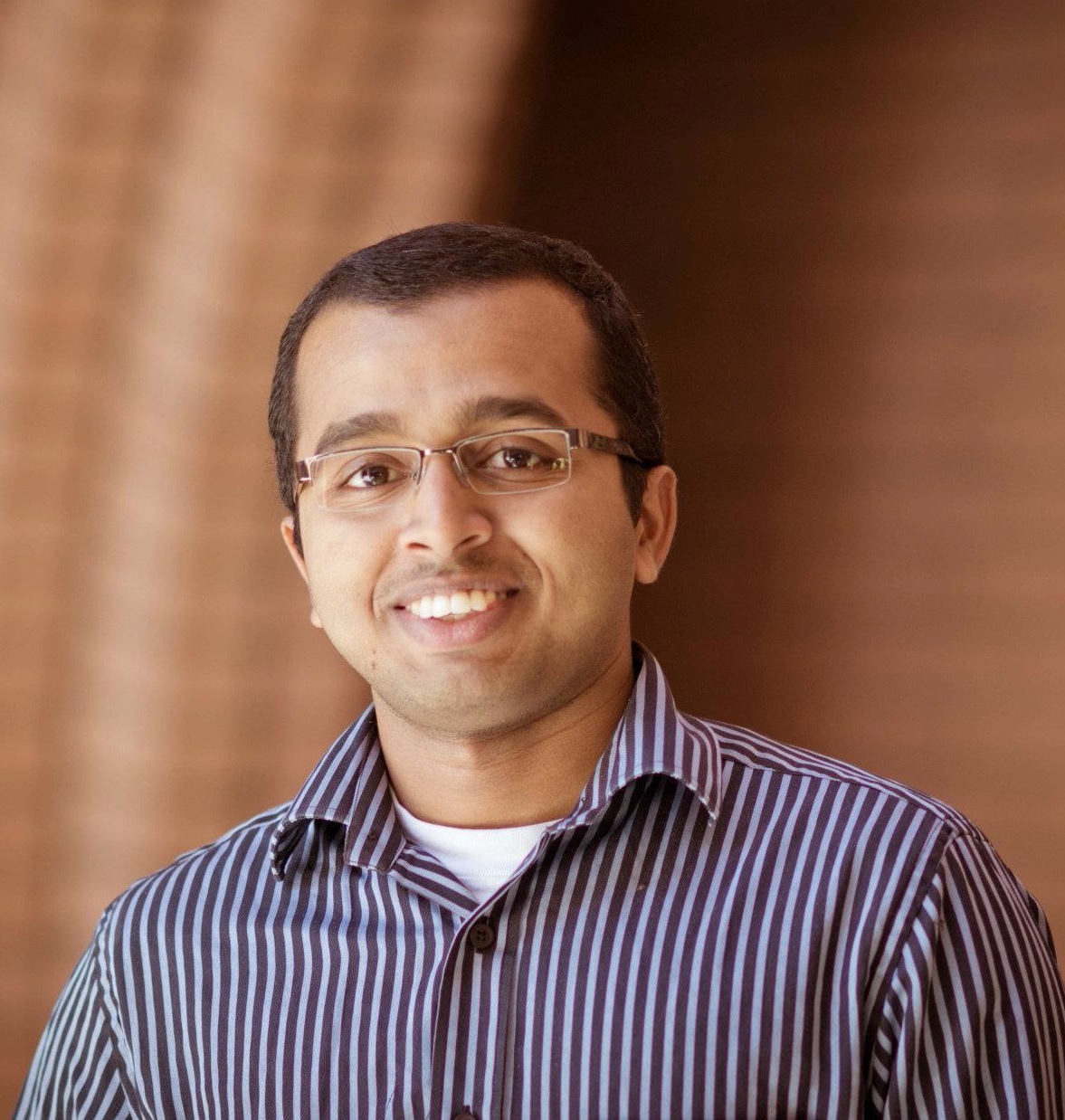}}]{Dileep  Kalathil}
is a postdoctoral scholar in the Department of Electrical Engineering and Computer Sciences at the University of California, Berkeley. He received his PhD from University of Southern California (USC) in 2014 where he won the best PhD Dissertation Prize in the USC Department of Electrical Engineering. He received an M.Tech from IIT Madras where he won the award for the best academic performance in the EE department. His research interests include sustainable energy systems, data driven optimization, online learning, stochastic control, and game theory.
\end{IEEEbiography}

\begin{IEEEbiography}[{\includegraphics[width=1in,height
=1.25in,clip,keepaspectratio]{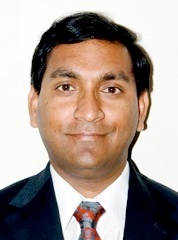}}]{Rahul Jain}
is an associate professor and the K. C. Dahlberg Early Career Chair in the EE department at the University of Southern California. He received his PhD in EECS and an MA in Statistics from the University of California, Berkeley, his B.Tech from IIT Kanpur. He is winner of numerous awards including the NSF CAREER award, an IBM Faculty award and the ONR Young Investigator award. His research interests span wireless communications, network economics and game theory, queueing theory, power systems and stochastic control theory.
\end{IEEEbiography}

\end{document}